\documentclass[letterpaper, 10 pt, conference]{ieeeconf}  % Comment this line out if you need a4paper

\makeatletter

\let\proof\@undefined
\let\endproof\@undefined
\makeatother

\IEEEoverridecommandlockouts                              % This command is only needed if 
\usepackage{amsmath,amsthm} 
\usepackage{amssymb}  % assumes amsmath package installed
\usepackage{graphicx}
\usepackage{calc}
\usepackage{subcaption}
\usepackage{url}
\usepackage{booktabs}
\let\labelindent\relax
\usepackage{enumitem,kantlipsum}
\usepackage{float}
\usepackage{graphics}
\newtheorem{theorem}{Theorem}[section]

\newtheorem{lemma}[theorem]{Lemma}

\theoremstyle{definition}
\newtheorem{definition}[theorem]{Definition}
\newtheorem{problem}[theorem]{Problem}

\theoremstyle{remark}

\renewcommand{\baselinestretch}{0.983}

\title{\LARGE \bf
Minimal $t-$spanning Primitive Sets: An MILP Formulation
}

\title{\LARGE \bf
Computing a Minimal Set of $t$-Spanning Motion Primitives\\ for Lattice Planners}

\urldef{\smith}\url{stephen.smith@uwaterloo.ca}
\urldef{\botros}\url{alexander.botros@uwaterloo.ca}

\author{Alexander Botros and Stephen L.\ Smith% <-this % stops a space
\thanks{This work is supported in part by the Natural Sciences and Engineering Research Council of Canada (NSERC)}% <-this % stops a space
\thanks{The authors are with the Department of Electrical and Computer Engineering,
        University of Waterloo, 200 University Ave W, Waterloo ON, Canada, N2L 3G1
    (\botros, \smith)
        }
}

\begin{document}

\maketitle
\thispagestyle{empty}
\pagestyle{empty}

%%%%%%%%%%%%%%%%%%%%%%%%%%%%%%%%%%%%%%%%%%%%%%%%%%%%%%%%%%%%%%%%%%%%%%%%%%%%%%%%
\begin{abstract}
In this paper we consider the problem of computing an optimal set of motion primitives for a lattice planner.  The objective we consider is to compute a minimal set of motion primitives that $t$-span a configuration space lattice.  A set of motion primitives $t$-span a lattice if, given a real number $t$ greater or equal to one, any configuration in the lattice can be reached via a sequence of motion primitives whose cost is no more than $t$ times the cost of the optimal path to that configuration.  Determining the smallest set of $t-$spanning motion primitives allows for quick traversal of a state lattice in the context of robotic motion planning, while maintaining a $t-$factor adherence to the theoretically optimal path.  While several heuristics exist to determine a $t$-spanning set of motion primitives, these are presented without guarantees on the size of the set relative to optimal. This paper provides a proof that the minimal $t-$spanning control set problem for a lattice defined over an arbitrary robot configuration space is NP-complete, and presents a compact mixed integer linear programming formulation to compute an optimal $t$-spanner. We show that solutions obtained by the mixed integer linear program have significantly fewer motion primitives than state of the art heuristic algorithms, and out perform a set of standard primitives used in robotic path planning.

\end{abstract}

%%%%%%%%%%%%%%%%%%%%%%%%%%%%%%%%%%%%%%%%%%%%%%%%%%%%%%%%%%%%%%%%%%%%%%%%%%%%%%%%
\section{Introduction}
Sampling based motion planning can typically be split into two methodologies: probabilistic sampling, and deterministic sampling.  In probabilistic sampling based motion planning, a set of $n$ samples is randomly selected over a configuration space and connections are made between ``close" samples.  A shortest path algorithm can then be run on the resulting graph.  Examples of probabilistic sampling based algorithms include Probabilistic Road Map (PRM)~\cite{kavraki1996probabilistic}, Rapidly Exploring Random Trees (RRT)~\cite{lavalle2006planning}, or RRT*~\cite{karaman2011sampling}.  Probabilistic sampling is attractive as it avoids explicitly constructing a discretization of the configuration space, and can provide any-time sub-optimal paths.  However, theoretical guarantees on the error of generated paths relative to optimal are typically asymptotic in nature due to the randomness of the samples.  
\par
In deterministic sampling based algorithms, on the other hand, a uniform discretization of the configuration space, often called a \emph{state lattice}, is explicitly constructed to form the vertices of a graph.  In \cite{pivtoraiko2005generating}, \cite{pivtoraiko2009differentially}, the authors define a state lattice as a graph whose vertices are uniformly distributed over a robot workspace, and whose edges are those tuples $(i,j)$ such that both $i$ and $j$ belong to the vertex set of the lattice, and there exists an admissible controller that brings $i$ to $j$.  The cost of any edge in the lattice is dictated by dynamics of the system in question and the choice of control.  The authors of~\cite{pivtoraiko2005generating} note that any complexity of the system due to its dynamics may be taken into account during the process of lattice construction by pre-computing the cost of each motion in the lattice given the dynamics of the system.  This removes the burden of accounting for complex dynamics during the actual process of robotic path planning. 
\par
In addition to reducing the complexities involved in robotic path planning, state lattices have also proved versatile in the problems that they can address.  For example, adaptations to standard rigid state lattices are widely used in autonomous driving.   The authors of \cite{rufli2010design}, \cite{mcnaughton2011motion} demonstrate state lattice adaptations made to account for the structured environments of urban roads.

\begin{figure}[t]
\centering

    \includegraphics[width = 0.95\columnwidth]{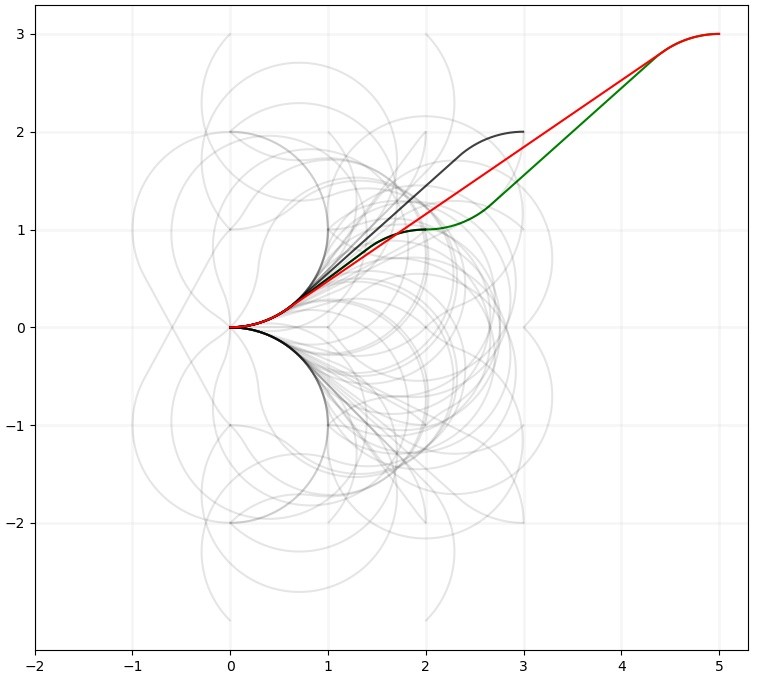}
    \caption{Example of motion planning using $t-$spanning motion primitives.  Vertex $(5,3,0)$ is reached optimally from $(0,0,0)$ via the red line, and is $1.015$-spanned by the light grey motion primitives.  The green line represents the path obtained using the motion primitives.  Black lines represent motion primitives used to reach $(5,3,0)$.}
    \label{TspanEx}
  
\end{figure}
State lattices can provide theoretical guarantees concerning the cost of proposed paths.  Using a connection radius growing logarithmically with the distance between between sampled lattice vertices, the authors of~\cite{janson2018deterministic} provide an upper bound on the error from optimal for a path produced by a deterministic sampling-based road map. The primary drawback with connecting lattice vertices in accordance with a connection radius  is the inclusion of redundant neighbours in the search space.  For example, for a square lattice in $\mathbb{R}^2$ with Euclidean distances, any-angle dynamics, and a connection radius $r\geq 2$, the vertex $(0,0)$ will have neighbours $(1,0)$ and $(2,0)$ while the vertex $(1,0)$ will also have $(2,0)$ as a neighbour.  When a shortest path algorithm is run, it will have to consider the path from $(1,0)$ to $(2,0)$ at least twice.  As a second example, consider the same lattice, and suppose that the connection radius $r$ is sufficiently large as to ensure that $(0,0)$ and $(4,1)$ are neighbors.  Then $(3,1)$ would also neighbour $(0,0)$, and the vertex $(4,1)$ could be reached by first passing through $(3,1)$.  The error on this decomposition of the direct path from $(0,0)$ to $(4,1)$ is given by
$$
\frac{\sqrt{3^2+1}+1}{\sqrt{4^2+1}} =1.0095,
$$
which may be sufficiently small ($0.95\%$) for the purposes of the user to remove $(4,1)$ as a neighbor of $(0,0)$.  Reducing the number of neighbours of a vertex results in a graph with fewer edges and faster shortest-path computations.  These redundancies motivate the following question: given a graph whose vertices are configurations of a mobile robot, and a real number $t\geq 1$, what is the smallest set of edges of the graph that guarantee that the cost-minimizing path from any vertex $i$ to any other vertex $j$ is no more than a factor of $t$ larger than cost of the optimal control between $i$ and $j$?  In other words, what is the smallest set of edges that $t-$span the vertices of the configuration space?   This set of edges, called $t$-spanning motion primitives, correspond to a set of admissible controllers that can be combined to navigate throughout the lattice to within a factor of $t$ of the optimal cost.  The notion of $t-$spanning motion primitives is similar to that of graph $t-$spanners first proposed in~\cite{peleg1989graph}.  An example of a $t-$spanning set of motion primitives is given in Figure~\ref{TspanEx}.

\par
 In~\cite{pivtoraiko2011kinodynamic}, the authors address the problem of computing a minimal set of $t$-spanning motion primitives for an arbitrary lattice.  We refer to this problem as the minimum $t$-spanning control set (MTSCS) problem.  For a Euclidean graph, attempting to determine such a minimum $t-$spanning set is known to be NP hard (see~\cite{carmi2013minimum}), and in~\cite{pivtoraiko2011kinodynamic}, the authors postulate that the same is true over an arbitrary lattice.  The authors of~\cite{pivtoraiko2011kinodynamic} provide a heuristic for the MTSCS problem, which while computationally efficient, does not have any known guarantees on the size of the set relative to optimal.
\par
The main contributions of this paper are twofold.  First, we provide a proof that the MTSCS problem is NP-complete for a lattice defined on an arbitrary robot workspace.  Second, a mixed integer linear programming (MILP) formulation of the MTSCS problem.  This MILP formulation contains only one integer variable for each edge in the lattice.  The compact formulation is achieved by establishing that the edges in an minimal lattice spanner must form a directed tree.  This MILP formulation constitutes the only non-brute force approach to optimally solving the MTSCS problem.  Though this MILP formulation does not scale to very large lattices, it can be solved offline and the solution will hold for any start-goal path planning instance in the same lattice.  Several numerical examples are provided in Section~\ref{SIMRES}.  These examples illustrate that minimal spanners can be computed for complex problems, and the resulting solutions are often significantly smaller than those found by the heuristic in~\cite{pivtoraiko2011kinodynamic}.  They also show the quality of paths constructed using the $t-$spanning control set in an A* search of paths in a lattice in the presence of obstacles.

\section{Background and Problem Statement}
In this section we formalize the notion of state lattices in terms of group theory and then define the minimum $t$-spanning control set problem.

\subsection{Preliminaries in Group Theory}
We begin by summarizing some of the concepts in geometric group theory outlined in~\cite{bullo2004geometric}, modifying some of the notation for our uses.  
\par
Consider any subgroup $G$ of the $d$-dimensional Special Euclidean group SE$(d)$.  By definition, the elements of $G$ are orientation-preserving Euclidean isometries of a rigid body in $d$-dimensional space. Let $s$ denote the identity element of $G$.  Any isometry $i$ in $G$ represents a possible configuration of a mobile robot.  However, $i$ can also be interpreted as a motion taking a mobile robot starting at $s$ to $i$.  As such, motions can be \textit{concatenated} to produce other motions.  The group operation of $G$, denoted $\cdot$ is defined as the left multiplication of elements of $G$. Thus, for $i,j\in G$, we may define another element $i\cdot j$ that is also in $G$ and represents the concatenation of the isometries $i$ and $j$. The isometry $i\cdot j$ is the motion that takes $s$ to $i$ by motion $i$, and then takes $i$ to $i\cdot j$ by motion $j$.

\subsection{State Lattices}

For a mobile robot, let $B_d(x)$ be defined as the set of all points in $\mathbb{R}^d$ occupied by the robot while its center of motion is at $x\in \mathbb{R}^d$.  
\begin{definition}[Robot Swath]
 Consider a mobile robot whose possible configurations are isometries in $G$.  Given $i,j\in G$ and an optimal steering function $u:[0,1]\rightarrow \mathbb{R}^d$ such that $u(0)=i, u(1)=j$, we define the \textit{swath} associated with the steering function $u$ between isometries $i$ and $j$ as
$$
\text{Swath}_u(i,j)=\{x\in \mathbb{R}^d : \exists \tau\in [0,1], x\in B_d(u(\tau)) \}.
$$
The set $\text{Swath}_u(i,j)$ represents the set of all vertices in $\mathbb{R}^d$ that the mobile robot will occupy as it traverses a path defined by the steering function from $i$ to $j$. 
\end{definition}
\begin{definition}[Valid Concatenation]
Consider isometries $i,j, p\in G$ of a mobile robot with optimal steering function $u$ that takes the robot starting at $i$ to $j$.  Suppose that $i\cdot p = j$, and let $W_d^k\subset \mathbb{R}^d$ denote a $d-$dimensional workspace bounded by a ball of radius $k$.  We say that the concatenation $i\cdot p = j$ is \textit{valid} if 
$$
\text{Swath}_u(i,j)\subseteq W_d^k.
$$
If $i\cdot p$ is a valid concatenation, we write $i\oplus p = i\cdot p = j$.  If $i\cdot j$ is not valid, we say that $i\oplus j$ is not defined.
\end{definition}
Given a workspace, we may determine the set of all valid concatenations of isometries in accordance with the above definition. However, given a set of concatenations $V$ that we wish to declare as valid, we may also define a workspace $W_d^k$.  This is done by assuming that the mobile robot in question has an admissible controller that takes $i$ to $j$ if and only if $i\cdot p=j$ is a valid concatenation (i.e., $i\oplus j\in V$).
\begin{definition}[Lattice]
\label{latticedef}
Let $(W_d^k, G, u)$ denote a workspace, group of isometries, and steering function respectively of a mobile robot.  Given a subset $B\subseteq G$, the \textit{lattice} generated by $B$ is defined as
\begin{equation}
\begin{split}
L^k(B) = \{j\in G\, : \, j = i_1\oplus i_2\oplus\dots \oplus i_m \text{ for some } \\ m\in \mathbb{N}, i_{1},...,i_{m}\in B\}.
\end{split}
\end{equation}
\end{definition}
  The lattice $L^k(B)$ is the set of elements, called \textit{vertices} of $G$ that can be obtained via valid concatenations of elements of $B$. 
  \par Definition \ref{latticedef} implies that there are two modes of construction of a lattice: lattice vertices may be fixed in $\mathbb{R}^d$ followed by the computation of steering functions from vertex to vertex, or a basic set of control primitives $B$ may be fixed and used to generate a lattice.

\par
Let $c:L^k(B)\rightarrow \mathbb{R}_{\geq 0}$ denote a cost function on $G$.  This cost function represents the cost of an optimal control taking $s$ to $p\in L^k(B)$.  Assuming that control costs are time invariant, if $i,j\in L^k(B)$ such that $i\cdot p = j$, (i.e., $p$ is the isometry taking vertex $i$ to vertex $j$) then the cost of the controller taking $i$ to $j$ is given by $c(p)$.  In this paper, we assume that  
\begin{enumerate}
    \item $c(p)\geq 0$, for all $p\in L^k(B)$,
    \item $c(p)=0$ implies that  $p=s,$ 
    \item if $i,j,k,p,q,r\in L^k(B)$ with $i\oplus p =j$, $j\oplus q=k$, and $i\oplus r=k$, then $ c(r)\leq c(p)+c(p).$
\end{enumerate}
If 1), 2), and 3) hold for a cost function $c$, we say that $c$ is an \textit{almost-metric}.  Note that these assumptions hold for control costs of mobile robots.  Observe that $c$ cannot properly be called a metric, as it lacks the symmetry requirement of a metric.  Figure~\ref{fig:Lattice} illustrates two examples of lattices.

\begin{figure}
    \centering
    \begin{subfigure}[b]{0.245\textwidth}
        \includegraphics[width=\textwidth]{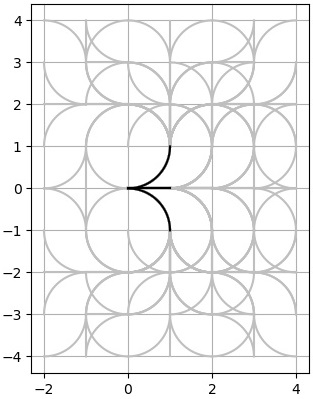}
        \label{fig:gull}
    \end{subfigure}
    ~ %add desired spacing between images, e. g. ~, \quad, \qquad, \hfill etc. 
    \begin{subfigure}[b]{0.22\textwidth}
        \includegraphics[width=\textwidth]{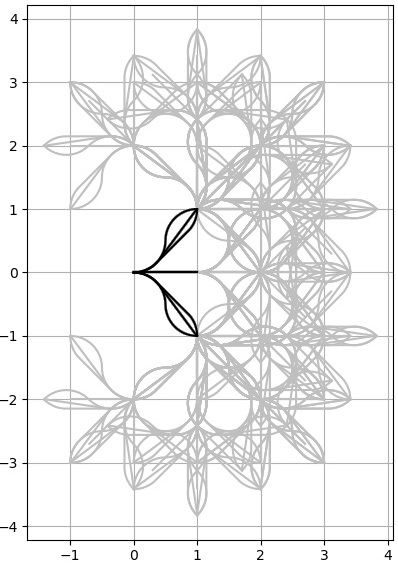}
        \label{fig:Lattddice}
    \end{subfigure}
    \caption{(Left) Partial lattice with three generators (black) in the group $G=\mathbb{R}^2\times SO(2)$, cost function given by Dubins' distance for turning radius $1$.  (Right) Partial lattice generated by six generators (black) in the group $G=\mathbb{R}^2\times SO(2)$.  Cost function $c$ is given by Dubins' distance for a minimum turning radius of $0.5$.}\label{fig:animals}
    \label{fig:Lattice}
\end{figure}

\iffalse
\begin{figure}[h]
\centering
\includegraphics[width=\linewidth]{Example_costs_in_lattice.pdf}
\caption{\footnotesize Example 1: Partial lattice generated by $B=\{(1,0), (0,1), (0,-1)\}$ in the group $G=\mathbb{R}^2$.  Cost function $c$ is given by Euclidean distances.  Example 2: Partial lattice generated by $B=\{(1,0,0), (1,\pm 1,0), (1,\pm 1,\pm\pi/2), (1,\pm 1,\pm\pi/4)\}$ in the group $G=\mathbb{R}^2\times SO(2)$.  Cost function $c$ is given by Dubin's distance for a minimum turning radius of $0.5$.\normalsize}
\label{fig:Lattice}
\end{figure}
\fi
The tuple $(L^k(B),c)$ induces a directed, weighted graph whose \textit{edges} are those tuples $(i,j)\in L^k(B) \times L^k(B)$ such that there exists $p\in L^k(B)$ with $i\oplus p = j$.  The \textit{cost} associated $(i,j)$ is $c_{ij}=c(p)$.  
\par
Let $E\subseteq L^k(B)$, and suppose that $j\in L^k(E)$.  We define a \textit{path} in $E$ to $j$, denoted $p^E(j)$, as a sequence of $m$ edges $(i_r, i_{r+1})$ such that $i_{r+1}=i_r\oplus p_r$ for $p_r\in E$ for all $r=1,\dots, m-1$, and that takes $s$ to $j$. The \textit{length} of the path $p^E(j)$ is 
\[
l=\sum_{r=1}^mc_{i_{r}i_{r+1}}.
\]
The \textit{distance} in $E$ to $j$, denoted $d^E(j)$, is the length of the length-minimizing path in $E$ to $j$.  
\par
For the remainder of this paper, the term path and the notation $p^E(j)$ will refer to a path in $E$ to $j$ of minimal length $d^E(j)$.  

\subsection{The Minimum $t$-Spanning Control Set Problem}

We are now ready to introduce the notion of $t$-reachability of a set of motion primitives in a lattice.

\begin{definition}[$t$-reachability]
Given $(L^k(B),c)$, a subset $E\subseteq L^k(B)$, and a real number $t\geq 1$, a vertex $j\in L^k(B)$ is \textit{$t$-reachable} from $E$ if 
\[
d^E(j)\leq t c(j).
\]
That is, $j$ is $t-$reachable from $E$ if the length of a shortest path $p^E(j)$ in $E$ to $j$ is no more than $t$ times the length of the shortest path from $s$ to $j$ in $L^k(B)$.  If every $y\in L^k(B)$ is $t-$reachable from $E$, we say that $E$ \textit{$t$-spans} the lattice $L^k(B)$.
\end{definition}
The Minimum $t-$Spanning Control Set problem is formulated as follows.
\begin{problem}[Minimum $t-$spanning Control Set Problem] 
\label{MCSP}
\textbf{Input:} A tuple $(L^k(B),c)$, and a real number $t\geq 1$.
\\
\textbf{Output:} A set $E\subseteq L^k(B)$ of minimal size that $t-$spans $L^k(B)$.
\end{problem}

Observe that $E$ is a set of $t$-spanning motion primitives for a mobile robot whose configuration space is given by $L^k(B)$.  That is, any configuration in the configurations space $L^k(B)$ may be decomposed into a sequence (path) of motions in $E$ such that the cost of any decomposition is no more than a factor of $t$ larger than the cost of the configuration.
\par
In the remainder of the paper we establish the hardness of the problem and then present a mixed integer linear Programming (MILP) formulation of Problem~\ref{MCSP}.

\section{Hardness of Computing MTSCS}
\label{sec:hardness}

In this section, we begin by characterizing the hardness of Problem~\ref{MCSP}, which serves to motivate the MILP formulation presented in the next section.   

\begin{theorem}
\label{NPHARDPROB}
Problem \ref{MCSP} is NP-hard.
\end{theorem}

\begin{figure}
    \centering
    \begin{subfigure}[b]{0.45\linewidth}
        \includegraphics[width = \linewidth]{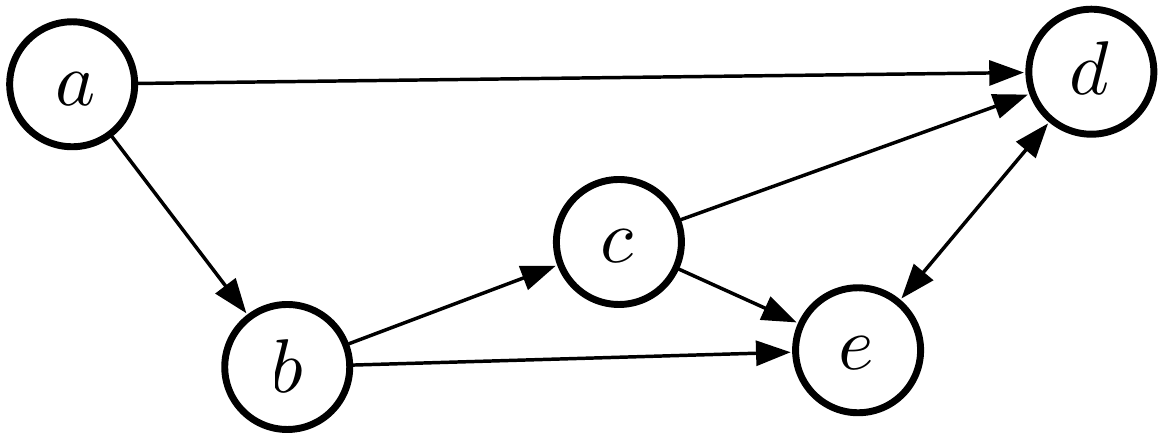}
        \label{fig:01graph_input}
        \caption{Graph $t$-spanner input.}
    \end{subfigure}
    \hfill
    \begin{subfigure}[b]{0.45\linewidth}
        \includegraphics[width = \linewidth]{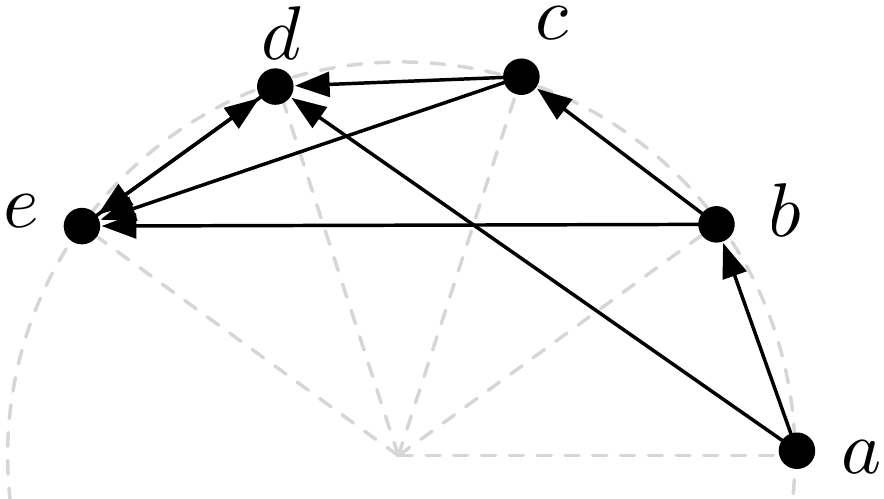}
        \caption{Arrangement with non-equal edge vectors.}
        \label{fig:02graph_embedding}
    \end{subfigure} \\[1em]
    \begin{subfigure}[b]{0.45\linewidth}
        \includegraphics[width = \linewidth]{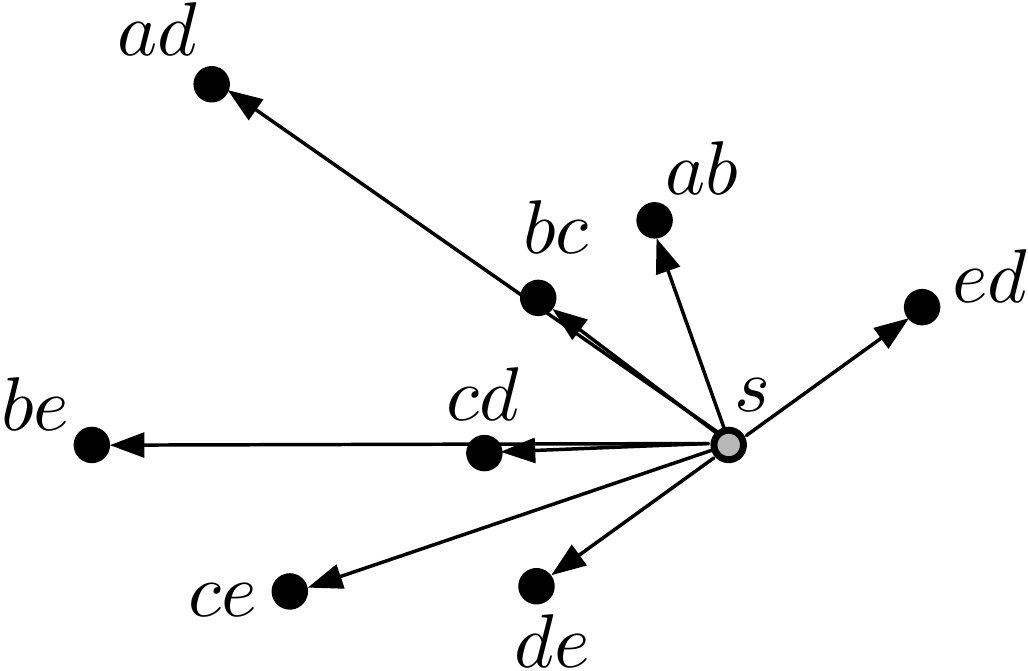}
        \caption{Lattice generating set $B$.}
        \label{fig:03lattice_generator}
    \end{subfigure}
    \hfill
    \begin{subfigure}[b]{0.45\linewidth}
        \includegraphics[width = \linewidth]{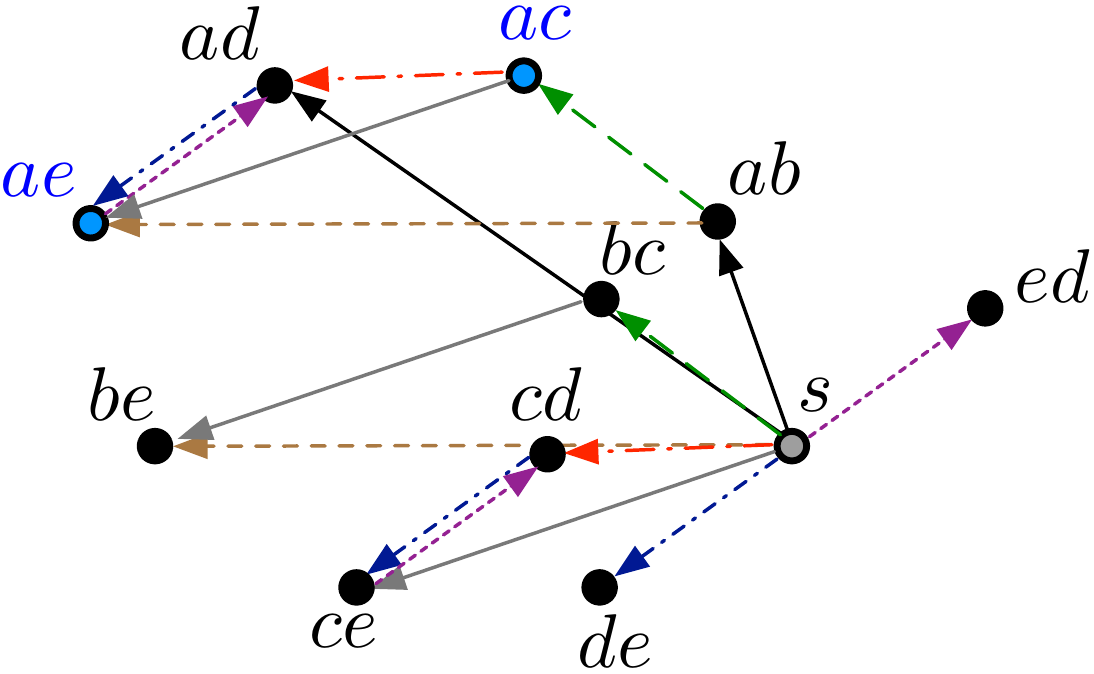}
        \caption{Complete lattice $L^k(B)$.}
        \label{fig:04lattice_full}
    \end{subfigure} \\[1em]
    \begin{subfigure}[b]{0.48\linewidth}
        \includegraphics[width = \linewidth]{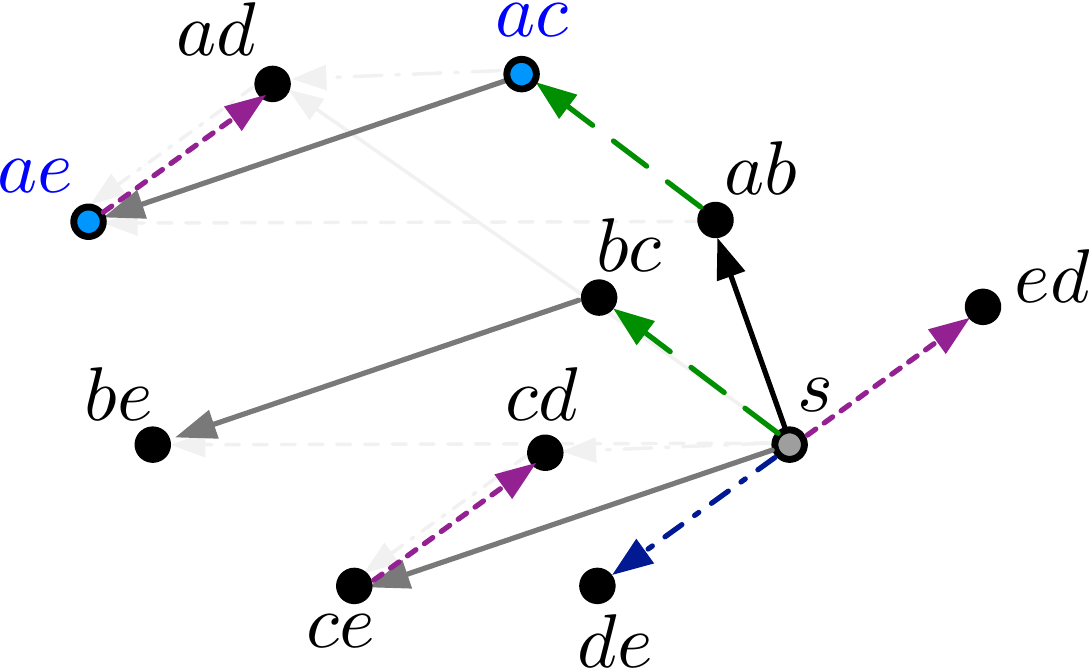} 
         \caption{Lattice $t-$spanner $Q'$.}
        \label{fig:05arborescence}
    \end{subfigure}
    \hfill
    \begin{subfigure}[b]{0.45\linewidth}
        \includegraphics[width = \linewidth]{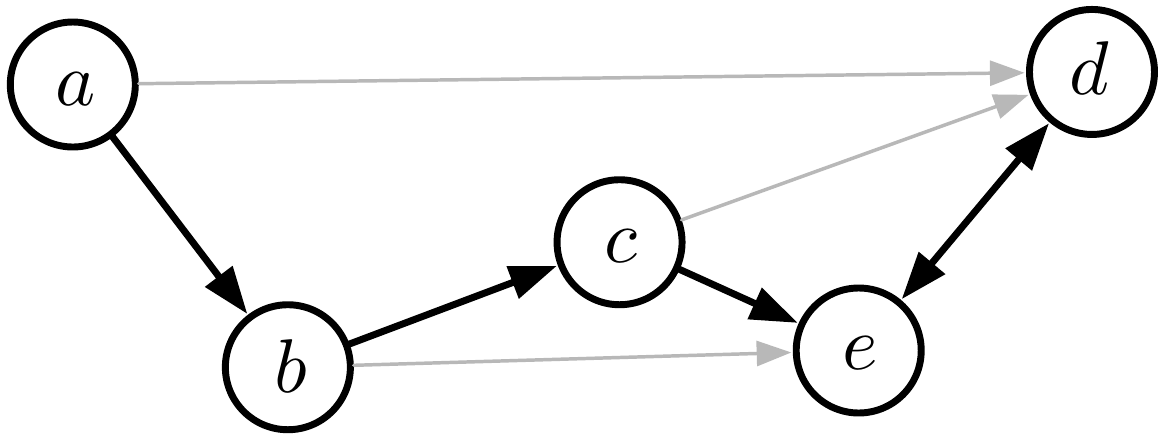}
            \caption{Graph $t$-spanner $Q$.}    
        \label{fig:06graph_spanner}
    \end{subfigure}
    \caption{Illustration of the steps in the reduction from metric graph $t$-spanner to Problem~\ref{MCSP}.}
    \label{fig:proof_illustration}
\end{figure}

\begin{proof}
     To show that Problem~\ref{MCSP} is NP-hard, we will construct a reduction from a metric graph $t-$spanner problem.  This problem is known to be NP-hard (see \cite{cai1994np}, \cite{carmi2013minimum}).   Given a directed graph $G=(V_G,E_G,w)$ with vertex set $V_G$, edge set $E_G$, edge weights $w$, and a value $t\geq 1$, we construct a lattice $L^k(B)$ in the underlying group $\mathbb{R}^{2}$ with almost-metric cost function $c$ and a real number $t'\geq 1$ that constitutes an instance of Problem~\ref{MCSP}.
     \par
     We begin by arranging the graph $G$ in the plane.  Vertices $a\in V_G$ will be represented as points in $a\in\mathbb{R}^2$, and edges $(a,b)\in E_G$ will correspond to vectors connecting $a$ and $b$.  The arrangement is done in such a way as to ensure that no two edge vectors are equal and can be performed in time polynomial in the size of the graph.
     \par
     To construct the arrangement, partition the upper half of the unit circle centered at the origin into $|V_G|$ equally spaced wedges, and place one vertex at the corner of each wedge starting at $(1,0)$.  For each edge $(a,b)\in E_G$, add a vector from $a$ to $b$ in the plane (see Figure~\ref{fig:02graph_embedding}).  Observe that no two edge vectors are equal.  Indeed, each edge vector corresponds to a non-diametral secant line of the unit circle. A circle in $\mathbb{R}^2$ possesses exactly two equal non-diametral secant lines appearing as mirror images about a diameter.  Therefore, two equal non-diametral secant lines cannot both occupy the upper half of the unit circle.
     \par
     We next construct the lattice in the underlying group $\mathbb{R}^2$ from the arrangement.  This is accomplished by first attaching each of the edge vectors to the origin.  For each edge vector $(a,b)$, declare a vertex $ab$ in the lattice generating set $B$ located at the tip of the edge vector (see Figure~\ref{fig:03lattice_generator}).  Observe that no two vertices in $B$ are co-located because no two edge vectors are equal.  Next, we construct the remaining lattice vertices and edges.  For each vertex $ab\in B$, determine the out edges in $G$ of the vertex $b$, say $(b,c), (b,e)$.  Concatenate (by vector addition) the vertices in $B$ corresponding with each out edge (say, $bc, be$) to the vertex $ab$.  The result of the concatenation is $ab\oplus bc = ac$.  If the vertex $ac$ already exists in the lattice,  proceed to the next out edge of $b$.  Otherwise, declare vertex $ac\in L^k(B)$.  The construction of the lattice is illustrated in Figure~\ref{fig:04lattice_full}.
     \par
     Finally, to each $ab\in L^k(B)$, define the cost $c(ab)$ as the cost of the minimal path in $G$ from $a$ to $b$.  Observe that this cost is not associated with the Euclidean norm of any vector used in the construction of $L^k(B)$.  Observe further, that in constructing the lattice $L^k(B)$, we have implicitly defined a workspace $W_2^k$ that admits concatenations $ab\cdot bc$ if and only if $(a,b), (b,c)\in E_G$, and admits $s\cdot ab$ if and only if $(a,b)\in E_G$.
     
     \par To prove the correctness of the reduction we show that if $Q'$ is a MTSCS of $L^k(B)$ (see Figure~\ref{fig:05arborescence}), then there exists a minimal $t-$spanner of $G$ of equal size. Observe that if $Q'$ is a minimal set of $t-$spanning vertices of $L^k(B)$, and $ab\in Q'$, then $(a,b)\in E_G$. Indeed, if $(a,b)\notin E_G$, then no path in $L^k(B)$ can involve a concatenation of any vertex in $L^k(B)$ with $ab$ by the definition of the workspace.  As such, $ab$ will not appear in any minimal set of $t-$spanning vertices.  
     \par
     Observe further, that any path $\{u_j, u_{j+1}\}_{j=1}^m$ of $m$ edges in $E_G$ corresponds uniquely to a path of vertices $\{u_ju_{j+1}\}_{j=1}^m$ in $L^k(B)$ and vice verse.  Moreover, the total cost of the path in $G$ is equal to the total cost of the path in $L^k(B)$.    
     \par
     Thus, the vertices in $Q'$ correspond to edges in $G$, and any path of vertices in $Q'$ corresponds to a path of edges in $G$ of equal cost.  Therefore, $Q'$ induces a set of edges $Q\subseteq E_G$ such that $(a,b)\in Q \iff ab\in Q'$, and where $Q$ is a $t-$spanner of $G$.  Therefore, the minimal $t-$spanner $G$ cannot have size greater than $|Q'|$.  Moreover, it follows that any set $Q$ that is a minimal $t-$spanner of $G$ induces a set $Q'$ of equal size comprised of vertices that $t-$span $L^k(B)$. Therefore, the minimal set of $t-$spanning motion primitives cannot have size larger than $Q$.  It follows then that $|Q'|=|Q|$ and a solution to our constructed instance of Problem \ref{MCSP} provides a minimal $t$-spanning set of edges $Q$ for the metric graph t-spanner problem.  This is summarized in Figures~\ref{fig:05arborescence}, and~\ref{fig:06graph_spanner}.

\end{proof}
\section{Computing a MTSCS}
In light of Theorem \ref{NPHARDPROB}, and assuming that $P\neq NP$, an efficient algorithm to exactly solve Problem~\ref{MCSP} does not exist. However, in what follows we formulate a MILP Problem~\ref{MCSP} that uses one just integer variable for each edge in $L^k(B)$.

\subsection{Properties of Minimal Spanning Control Sets}

For any vertex $p\in L^k(B)$, let
$$
S_p=\{(i,j) \ : \ i,j\in L^k(B), \  i\oplus p =j\}.
$$  
That is, $S_p$ is the set of all pairs $(i,j)$ such that $p$ takes $i$ to $j$ and $i\cdot p$ is a valid concatenation. The set of all edges in the graph $L^k(B)$ is given by
$$
\mathcal{E}=\bigcup_{p\in L^k(B)-\{s\}}S_p.
$$
\par
A naive approach to developing a MILP would be to define $|L^k(B)|\cdot|\mathcal{E}|$ integer variables $x_{pij}$, each which takes the value 0 if $(i,j)\in S_p$ and $p\not\in E$, and value 1 if $p\in E$.  However, the number of required integer variables in the MILP can be reduced to $|\mathcal{E}|$ via a graph theoretic approach to Problem~\ref{MCSP}.
This approach is motivated by the following definition and lemma:
\begin{definition}[Arborescence]
In accordance with Theorem 2.5 of~\cite{korte2012combinatorial}, a graph $T$ with a vertex $s$ is an arborescence rooted at $s$ if every vertex in $T$ is reachable from $s$, but deleting any edge in $T$ destroys this property.
\end{definition}
\begin{lemma}
\label{TREELEM}
Let $E\subseteq L^k(B)$ be a solution to Problem 1. We construct a graph $T=(V_T, E_T)$ whose vertices $V_T$ are those vertices in $L^k(B)$, and whose edges $E_T$ are defined as follows: let
$$
T' = \bigcup_{i\in L^k(B)-\{s\}}p^E(i).
$$
For each $i\in L^k(B)-\{s\}$, if $T'$ contains two paths $p_1, p_2$ to $i$, remove the last edge in $p_2$ from $T'$.  Let the remaining edges be the set $E_T$.  Then $T$ is an arborescence rooted at $s$, and for each $i\in L^k(B)$, the value $d^E(i)$ is the length of the path in $T$ to $i$.
\end{lemma}
\begin{proof}
Observe that $L^k(E) = L^k(B)$ since $E$ solves Problem~\ref{MCSP}.  Therefore, if $j$ is a vertex in $T$, then it must be reachable from $s$.  Observe further, that for each $j\in V_T$, there exists a unique path in $T$ to $j$.  Indeed, if there were two paths $p_1, p_2$ of edges in $E_T$ to $j$, then $p_1, p_2$ would be paths in the edge set $T'$ because $E_T\subseteq T'$.  However, upon construction of $E_T$ from $T'$, the last edge of $p_2$ would be removed, implying that $p_2$ could not be a path in $E_T$ to $j$. 
\par
Suppose that an edge $(i,j)$ is removed from $E_T$.  Then by the definition of $T$, there must exist some $r\in V_T$ with $(i,j)$ on the unique path in $T$ to $r$.  Therefore, if $(i,j)$ is removed from $E_T$, then there is no path of edges in $E_T-\{(i,j)\}$ from $s$ to $r$, which implies that $r$ is not reachable from $s$.  Therefore, $T$ is an arborescence.
\par
It will now be shown that the length of the path in $T$ to any vertex $i\in V_T$ is $d^E(B)$.  Recall that $p^E(i)$ is defined as a path of minimal length in $E$ to $i$, and the length of this path is $d^E(i)$.  Therefore, for any $i\in V_T$, there exists one path $p^E(i)\subset E_T$ to $i$, and the length of this path is $d^E(i)$. This implies that the distance in $T$ to $i$ is $d^E(i).$
\end{proof}
Lemma \ref{TREELEM} implies that if $E$ is a minimum $t-$spanning control set of $L^k(B)$, then there is a corresponding arborescence $T$ whose vertices are those vertices of $L^k(B)$, whose edges $(i,j)$ are members of $S_p$ for some $p\in E$, and in which the cost of the path from $s$ to any vertex $i\in T$ is no more than $tc_i$. 
\par
Suppose now that $|L^k(B)|=n$ and that all the vertices are enumerated as $1,2,...,n$ with $s=1$.  For any set $E\subseteq L^k(B)$ such that $L^k(B)\subseteq L^k(E)$, define $n-1$ decision variables $y_p, p=2,...,n $ as
$$
y_p = \begin{cases}
1, \ \text{if } p\in E\\
0, \ \text{otherwise}.
\end{cases}
$$
For each $(i,j)\in \mathcal{E}$, let  
$$
x_{ij}=\begin{cases}
1 \ &\text{if } (i,j)\in T \\
0 \ &\text{otherwise}. 
\end{cases}
$$
Let $z_i$ denote the length of the path in the tree $T$ to vertex $i$ for any $i\in L^k(B)$, and let $L'=L^k(B)-\{s\}$.  

\subsection{High-Level Description of Optimization}

We begin by providing an high-level description of the objective and problem constraints, followed by a precise MILP formulation.  To solve Problem~\ref{MCSP} our objective is to minimize $|E|$ subject to the following constraints:
\begin{description}
\item[Usable Edge Criteria:] For any $p\in L'$, if $y_p=1$, then $p\in E$.  Therefore, for any $(i,j)\in S_p$, the variable $x_{ij}$ can take either the value 0 or 1.  On the other hand, if $y_p=0$, then $x_{ij}=0$ for all $(i,j)\in S_p$, implying that $(i,j)$ may not appear in $T$.

\item[Cost Continuity Criteria:]
If $x_{ij}=1$, for any $(i,j)\in \mathcal{E}$, then the path in the arborescence $T$ from $s$ to $j$ contains the vertex $i$.  Therefore, it must hold that
\[
z_j = z_i + c_{ij}.
\]
\item[$t-$Spanning Criteria:]
The length of the path in $T$ to any vertex $j\in L'$ can be no more than $t$ times the length of the direct edge from $s$ to $j$.  That is,
\[
z_j \leq tc_j, \ \forall j\in L'.
\]
\item[Arborescence Criteria:]
The set $T$ must be an arborescence.

\end{description}

\subsection{MILP Formulation}

The constraints listed above can be encoded as the following MILP.  
\begin{subequations}
\label{MILP1}
\begin{align}
\label{OF}
&\min \sum_{p=2}^n y_p\\
&s.t.  \\
\label{XY}
&x_{ij}-y_p\leq 0,  \ &&\forall(i,j)\in S_p, \ \forall p\in L'\\
\label{COST}
&z_i+c_{ij}-z_j\leq M_{ij}(1-x_{ij}),  \ &&\forall (i,j)\in \mathcal{E}\\
\label{COST3}
&z_j\leq tc_j, \  &&\forall j\in L'\\
\label{TREE}
&\sum_{i\in L'}x_{ij}=1, \ &&\forall j\in L'   \\
\label{Boolx2}
&x_{ij}\in \{0,1\}, \ &&\forall (i,j)\in\mathcal{E}\\
\label{Booly}
&y_p\in[0,1], \ &&\forall p\in L',
\end{align}
\end{subequations}
where $M_{ij}=tc_i+c_{ij}-c_j$.  The constraints of (\ref{MILP1}) are explained as follows.

\textbf{Constraint (\ref{XY}):} If $p\not\in E$, then $y_p=0$ by definition.  Therefore, (\ref{XY}) requires that $x_{ij}=0$ for all $(i,j)\in S_p$. Alternatively, if $p\in L'$, then $y_p=1$, and $x_{ij}$ is free to take values $1$ or $0$ for any $(i,j)\in S_p$.  Thus constraint (\ref{XY}) is equivalent to the Usable Edge Criteria.

\textbf{Constraint (\ref{COST}):}
Constraint (\ref{COST}) encodes Cost Continuity Criteria. It takes a similar form to~\cite[Equation~(2.7)]{desrosiers1995time}.  Begin by noting that $M_{ij}\geq 0$ for all $(i,j)\in E$.  Indeed, for any $t\geq 1$,
\[
M_{ij}\geq c_i+c_{ij}-c_j,
\]
and $c_i+c_{ij}\geq c_j$ by the definition of an almost-metric.  Replacing the definition of $M_{ij}$ in (\ref{COST}) yields
\begin{equation}
\label{cc}
z_i+c_{ij}-z_j\leq (tc_i+c_{ij}-c_j)(1-x_{ij}).
\end{equation}
If $x_{ij}=1$, then (\ref{cc}) reduces to $z_j\geq z_i+c_{ij}$.  If, however, $x_{ij}=0$, then (\ref{cc}) reduces to $z_i-z_j\leq tc_i-c_j$ which holds trivially by constraint (\ref{COST3}) and by noting that $z_j\geq c_j, \forall j\in L$ by the definition of an almost-metric.  

\textbf{Constraint (\ref{TREE}):} Constraint (\ref{TREE}) together with constraint (\ref{COST}) yield the Arborescence Criteria.  Indeed, by Theorem 2.5 of~\cite{korte2012combinatorial}, $T$ is an arborescence rooted at $s$ if every vertex in $T$ other than $s$ has exactly one incoming edge, and $T$ contains no cycles.  The constraint (\ref{TREE}) ensures that every vertex in $L'$, which is the set of all vertices in $T$ other than $s$, has exactly one incoming edge, while constraint (\ref{COST}) ensures that $T$ has no cycles.  To see why this is true, suppose that a cycle existed in $T$, and that this cycle contained vertex $i\in L'$.  Suppose that this cycle is represented as
\[
i\rightarrow j\rightarrow \dots\rightarrow k\rightarrow  i.
\]
Note that (\ref{COST}) and the definition of an almost-metric imply that $z_i<z_j$ for any $(i,j)\in \mathcal{E}$.  Therefore, 
\[
z_i<z_j<\dots<z_k<z_i,
\]
which is a contradiction.

\textbf{Constraint (\ref{Booly}):} The variable $y_p$ is a decision variable, and therefore should take values in $\{0,1\}$ for all $p\in L'$.  However, the integrality constraint on $y_p$ may be relaxed to (\ref{Booly}).  Indeed, suppose that for any $p\in L'$, there exists an edge $(i,j)\in S_p$ such that $x_{ij}=1$.  Then, by (\ref{XY}), $y_p\geq 1$ which implies by (\ref{Booly}), that $y_p=1$.  If, on the other hand, there does not exist $(i,j)\in S_p$ with $x_{ij}=1$, then $y_p$ is free to take values in $[0,1]$.  Therefore, $y_p=0$ as the objective function seeks to minimize the sum of $y_p$ over $p$.
\\
\\
Observe that the NP-hardness of Problem~\ref{MCSP} is proved in Theorem~\ref{NPHARDPROB}, while a reduction from Problem~\ref{MCSP} to an MILP is given in \ref{MILP1}.  Therefore, a reduction of the decision version can be constructed both from and to NP-complete problems, which implies that the decision version of Problem~\ref{MCSP} is NP-complete.
\par
In the next section, the importance of minimal $t-$spanning control sets to the field of motion planning is illustrated by way of numerical examples.

\section{Simulation Results}
\label{SIMRES}
We implemented the  MILP presented in (\ref{MILP1}) in Python 3.6, and solved using Gurobi.   In this section, we begin by comparing the size of the set of $t-$spanning motion primitives determined by (\ref{MILP1}) with those obtained from the sub-optimal algorithm presented in~\cite{pivtoraiko2011kinodynamic}.  We also compare paths generated using primitives computed here with those standard primitives appearing in~\cite{SBPL}.  We conclude with an brief investigation of motion primitives in grid-based path planning.

\subsection{Comparison to Existing Primitive Generators}

 The first example we consider is the lattice
\[
L_1 = (\mathbb{Z}^2\cap [0,k]\times[-k,k])\times \{0, \pi/2, \pi, 3\pi/2\}.
\]
The lattice $L_1$ is generated by the set $B=\{(1,0,0), (1,1,\pi/2), (1,-1, 3\pi/2)\}$.  The workspace $W_d^k$ is defined as $[0,k]\times [-k,k]$.  The mobile robot is assumed to be a single point in $\mathbb{R}^2$, and swaths and costs are defined by Dubins' paths and path lengths, respectively.
\par
For the lattice $L_1$, the size of the MTSCS $|E^*|$ for varying minimal turning radii $R$, and values of $t$ and $k$ were computed using the MILP in (\ref{MILP1}).  The sizes of these control sets were compared with those obtained by employing the heuristic algorithm proposed in~\cite{pivtoraiko2011kinodynamic}.  Table~\ref{4HEAD} summarizes the findings.

\begin{table}
    \centering
     \begin{tabular}{@{} l c c c c c c c c @{}}
    \toprule
       & \multicolumn{2}{c}{$k = 3$} &  & \multicolumn{2}{c}{$k = 4$} && \multicolumn{2}{c}{$k = 7$} \\
       \cmidrule{2-3} \cmidrule{5-6} \cmidrule{8-9}
      & $|E^*|$ & $|E_P|$ && $|E^*|$ & $|E_P|$ && $|E^*|$ & $|E_P|$ \\
      \midrule
      $R=0.5$ & & & & & & & & \\
      \cmidrule{1-1}
      $t= 1.01$ & 70 & 70 && 92 & 94 && 124 & 137 \\
      $t = 1.5$ & 9  & 9  && 9  & 9  && 9   & 9   \\
      $t = 3$   & 6  & 6  && 6  & 6  && 6   & 6 \\
      \addlinespace
      $R=2$ & & & & & & & & \\
      \cmidrule{1-1}
      $t= 1.01$ & 75 & 75 && 90 & 92 && 128 & 132 \\
      $t = 1.5$ & 12  & 20 && 13  & 19  && 11   & 19   \\
      $t = 3$   & 7  & 7  && 10  & 10  && 10   & 11 \\
      \addlinespace
      $R=4$ & & & & & & & & \\
      \cmidrule{1-1}
      $t= 1.01$ & 69 & 69 && 102 & 102 && 223 & 231 \\
      $t = 1.5$ & 16  & 16 && 16  & 24  && 19   & 40   \\
      $t = 3$   & 3  & 5  && 7  & 7  && 13   & 14 \\
        \bottomrule
       \end{tabular}
        \caption{Results for lattice $L_1$.}
    \label{4HEAD}
\end{table}

The error $|E_P|-|E^*|$ can be as low as 0 in some cases.  However, for certain values of $R,t,k$, the error can be larger than $|E^*|$.  For example, when $t=1.5, R=4, k=7$, the algorithm proposed in~\cite{pivtoraiko2009differentially} returns a control set that is more than twice as large as the optimal. This will greatly slow any path planning algorithm that uses the control set $E_P$.   

For the second example we consider the lattice
\[
L_2 = (\mathbb{Z}^2\cap [0,k]\times[-k,k])\times \big\{i\frac{\pi}{4}\big\}_{i=0}^7.
\]
Observe that $L_2$ is the eight heading (cardinal and ordinal) version of $L_1$.  Table ~\ref{tab:8head} summarizes the findings for $L_2$.

   %%%%%%%%%%%%%%%%%%%%%%%%%
  \begin{table}
      \centering
       \begin{tabular}{@{} l c c c c c @{}}
      \toprule
         & \multicolumn{2}{c}{$k = 3$} &  & \multicolumn{2}{c}{$k = 4$} \\
         \cmidrule{2-3} \cmidrule{5-6}
        & $|E^*|$ & $|E_P|$ && $|E^*|$ & $|E_P|$ \\
        \midrule
        $R=0.5$ & & & & & \\
        \cmidrule{1-1}
        $t= 1.01$ & 154 & 154 && 196 & 198 \\
        $t = 1.5$ & 19  & 19  && 19  & 19    \\
        $t = 3$   & 10  & 12  && 10  & 12  \\
        \addlinespace
        $R=2$ & & & & &  \\
        \cmidrule{1-1}
        $t= 1.01$ & 159 & 159 && 214 & 222 \\
        $t = 1.5$ & 34  & 50 && 31  & 49   \\
        $t = 3$   & 15  & 22  && 19  & 23 \\
        \addlinespace
        $R=4$ & & & & & \\
        \cmidrule{1-1}
        $t= 1.01$ & 147 & 147 && 226 & 232  \\
        $t = 1.5$ & 44  & 44 && 50  & 68   \\
        $t = 3$   & 5  & 18  && 11  & 20 \\
          \bottomrule
         \end{tabular}
         \caption{Results for lattice $L_2$.}
          \label{tab:8head}
  \end{table}

Observe that for $k=3, R=4, t=3$, the size of the set $E_P$ is over 3 times that of $E^*$.  Tables~\ref{4HEAD}, and \ref{tab:8head} imply that there are several instances of mobile robot and desired configurations of the robot which are detrimental to the algorithm proposed in~\cite{pivtoraiko2011kinodynamic}.   This algorithm, which represents the state of the art on the subject of minimal $t-$spanning control set generation, produces control sets that are at times several times larger than the minimal $t-$spanning control set.  Moreover, it is not obvious, given an instance of problem \ref{MCSP}, when the error $|E^*|-|E_P|$ will be small.  The runtime to compute each MILP-based spanning set in Tables \ref{4HEAD}, \ref{tab:8head}  ranged from a few seconds to on the order on an hour.

\subsection{Path Planning Comparison}

A standard set of motion primitives can be found in the ROS package SBPL~\cite{SBPL}.  For a minimum turning radius $R=0.5$, 8 headings (cardinal and ordinal), and cost function given by the Dubins' distance, 40 start-goal path planning problems were created.  These problems are comprised of a randomly generated goal location and set of obstacles.  Each problem was solved using the same A* implementation.  This A* algorithm operates by expanding search nodes starting at $s=(0,0,0)$.  The neighbors of each search node are are determined by concatenating the node with each of the primitives.  A concatenation $i\cdot j$ is deemed valid if and only if the $x$ and $y$ coordinates of $i\cdot j$ are both integers. 
\begin{figure}
    \centering
    \begin{subfigure}[b]{0.5\textwidth}
        \includegraphics[width=\textwidth]{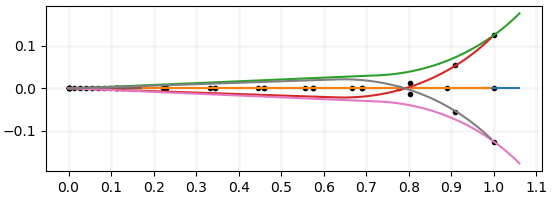}
        \caption{SBPL primitives.}
        \label{ROSPRIMS}
    \end{subfigure}\\[1em]
    ~ %add desired spacing between images, e. g. ~, \quad, \qquad, \hfill etc. 
    \begin{subfigure}[b]{0.3\textwidth}
        \includegraphics[width=\textwidth]{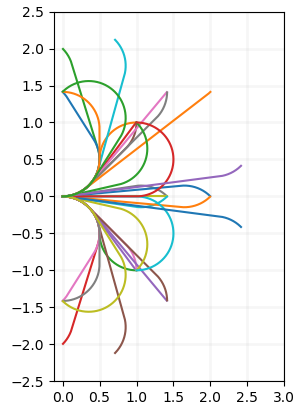}
        \caption{MILP primitives.}
        \label{MILPPRIMS}
    \end{subfigure}
   \caption{(a): Motion primitives found in~\cite{SBPL}.  (b): Motion primitives determined from MILP (\ref{MILP1})}
    \label{PRIMSEX}
\end{figure}
The primitives used are illustrated in Figure~\ref{PRIMSEX}. Figure~\ref{ROSPRIMS} illustrates the 8 primitives proposed in~\cite{SBPL} (note that some primitives are not visible as they are short).  The 33 primitives in \ref{MILPPRIMS} were obtained by solving the MILP in (\ref{MILP1}) for the lattice $L^k(B)$ generated by
\begin{equation}
    \begin{split}
   B_{\text{MILP}} = \{&(1,0,0), (1,1,0), (1, -1, 0), (1,1,\pi/4),\\
   &(1,-1,7\pi/4), (1,1,\pi/2), (1,-1, 3\pi/2),\\
   &(\sqrt{2}, 0, 7\pi/4), (\sqrt{2}, 0, 0), (\sqrt{2}, 0,\pi/4)  \}
    \end{split},
\end{equation}
and a value of $t=1.4$.  The primitives obtained from the MILP consistently outperformed the standard primitives. We use a ratio of run times ($T_{\text{MILP}}/T_{\text{ROS}}$) of the A* algorithm as a metric for comparing the time efficiency of the two primitive sets, while a ratio of path lengths ($L_{\text{MILP}}/L_{\text{ROS}}$) provides a basis for comparison of performance.  Of the 40 randomly generated maps, the average length and time ratios are
\[
\left(\frac{L_{\text{MILP}}}{L_{\text{ROS}}}\right)_{\text{avg}} = 0.897, \; \text{and} \;\ \left(\frac{T_{\text{MILP}}}{T_{\text{ROS}}}\right)_{\text{avg}} = 0.267.
\]
 The average path length using the MILP-obtained primitives was $\approx 90\%$ as long as those for the ROS package primitives and took $\approx 27\%$ of the time to calculate.  In fact, in each of the 40 maps the MILP-obtained primitives consistently outperformed the ROS package primitives for both metrics.  Figure~\ref{PAHEX} presents three example maps and paths.  We notice that in addition to shorter faster solutions, the MILP-generated primitives also result in smoother paths with fewer turns.  This fact will facilitate any further smoothing that is required.

\begin{figure}
    \centering
    \begin{subfigure}[b]{0.45\linewidth}
        \includegraphics[width=\textwidth]{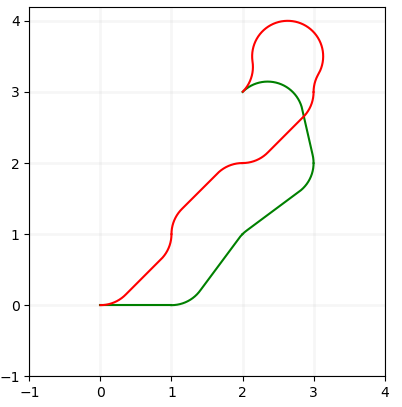}
        \caption{Zero obstacles.}
        \label{Zero obstacles}
    \end{subfigure}
    ~ %add desired spacing between images, e. g. ~, \quad, \qquad, \hfill etc. 
    \begin{subfigure}[b]{0.45\linewidth}
        \includegraphics[width=\textwidth]{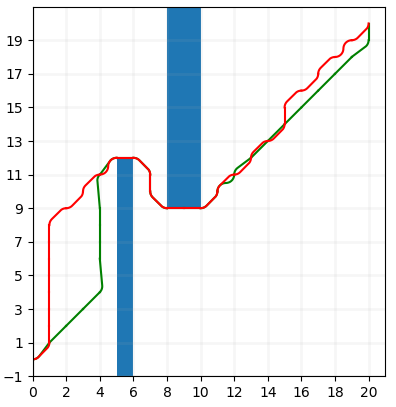}
        \caption{Two obstacles.}
        \label{Tow obstacles}
    \end{subfigure}\\[1em]
    \begin{subfigure}[b]{0.45\linewidth}
        \includegraphics[width=\textwidth]{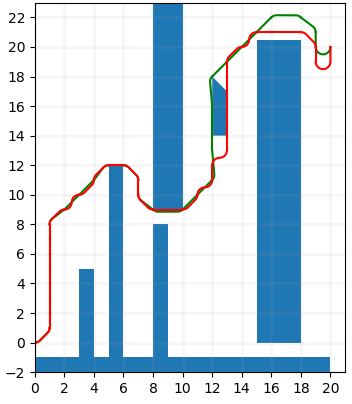}
        \caption{Seven obstacles.}
        \label{Seven obstacles}
    \end{subfigure}
   \caption{Example paths for two sets of motion primitives. Red lines represent paths generated by standard motion primitives (see Figure \ref{ROSPRIMS}).  Green lines represent paths generated using MILP-obtained motion primitives (see Figure \ref{MILPPRIMS}).}
    \label{PAHEX}
\end{figure}

\subsection{Motion Primitives Euclidean Lattices}

Another common problem is that of planning shortest paths in occupancy grids.  The error between the shortest grid path and the optimal path depends on the number of neighbors considered for each gridpoint in the search.  In~\cite[Figure 4]{nash2013any}, the authors present 4 and 8 neighbor grids in two dimensions, and 6 and 26 neighbor grids for three dimensions.  They also provide error results in the form of $t$-values for these grid choices~\cite[Table 1]{nash2013any}.  For example, in the two dimensions, the 4 neighbor grid provides a $t$ of $\approx 1.4$, while the 8 neighbor grid provides a $t$ of $\approx 1.08$.  Given that a $k\times k$ cubic grid in $d$ dimensions can be modelled as a lattice of the form $L=\mathbb{Z}^d\cap [-k, k]^d$, generated by the canonical basis of $\mathbb{R}^d$, we can, for a given $t\geq 1$, determine a minimal set of $t-$spanning motion primitives using the MILP in (\ref{MILP1}).  In Figure~\ref{EUC}, we show the first quadrant of the neighbors needed to achieve $t$ values of 1.0274, which results in a 16 neighbor grid, 1.0131, which results in a 24 neighbor grid, and 1.0124, which results in a 36 neighbor grid. These solutions extend the results in~\cite{nash2013any}.
\begin{figure}[H]
    \centering
    \includegraphics[width=0.99\columnwidth]{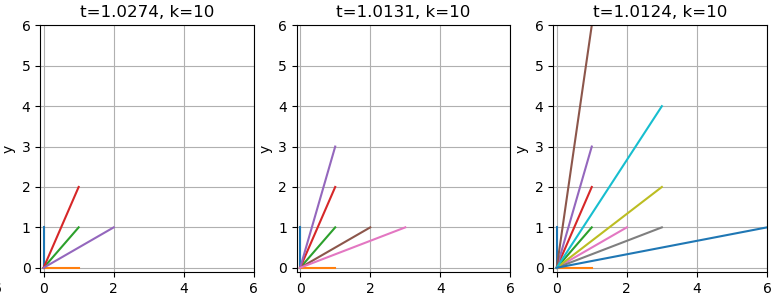}
   \caption{Euclidean lattice minimal $t-$spanners in 2D and for different values of $t$.}
    \label{EUC}
\end{figure}

\section{Conclusions and Future Work}
The numerical examples presented in Section~\ref{SIMRES} illustrate the importance of minimal sets of $t-$spanning motion primitives in robotic motion planning.  The MILP formulated in (\ref{MILP1}) represents the only known non-brute force approach to calculating \textit{exact} minimal $t-$spanning motion primitives. Though solutions to (\ref{MILP1}) in general cannot be obtained in time polynomial in the size of the input lattice, motion primitives are generally calculated once, offline.  
\par
Observe that while minimal $t-$spanning motion primitives for a given lattice may be calculated using the MILP formulation in (\ref{MILP1}), the choice of lattice appears to be equally important to the time and length efficiency of path planning problems.  The correct choice of lattice for a given mobile robot is a subject of future work. 
\par
There are lattices for which a solution to Problem~\ref{MCSP} may be efficiently obtained.  Indeed, it can be shown that Problem \ref{MCSP} for a Euclidean lattice with convex workspace is efficiently solvable regardless of the dimension.  Observe that the proof of Theorem \ref{NPHARDPROB}, in which it is established that Problem \ref{NPHARDPROB} is NP-hard, relies heavily on the potential non-convexity of the lattice workspace.  The conditions under which Problem~\ref{MCSP} can be efficiently solved is still an open question.

\bibliographystyle{IEEEtran}

%\bibliography{references}

\end{document}